%% file: bingham_so3_learning.tex
\documentclass[conference]{IEEEtran}

\input{packages}
\input{notation}

\begin{document}

\title{A Smooth Representation of Belief over $\LieGroupSO{3}$ for Deep Rotation Learning with Uncertainty}

\author{\authorblockN{Valentin Peretroukhin,$^{1,3}$ Matthew Giamou,$^1$ David M. Rosen,$^2$ W. Nicholas Greene,$^3$ \\ Nicholas Roy,$^3$ and Jonathan Kelly$^1$}
\authorblockA{${^1}$Institute for Aerospace Studies, University of Toronto;\\
${^2}$Laboratory for Information and Decision Systems, \\ ${^3}$Computer Science \& Artificial Intelligence Laboratory, Massachusetts Institute of Technology}}

\maketitle

\begin{abstract}
Accurate rotation estimation is at the heart of robot perception tasks such as visual odometry and object pose estimation. Deep neural networks have provided a new way to perform these tasks, and the choice of rotation representation is an important part of network design. In this work, we present a novel symmetric matrix representation of the 3D rotation group, $\LieGroupSO{3}$,  with two important properties that make it particularly suitable for learned models: (1) it satisfies a smoothness property that improves convergence and generalization when regressing large rotation targets, and (2) it encodes a symmetric Bingham belief over the space of unit quaternions, permitting the training of uncertainty-aware models. We empirically validate the benefits of our formulation by training deep neural rotation regressors  on two data modalities. First, we use synthetic point-cloud data to show that our representation leads to superior predictive accuracy over existing representations for arbitrary rotation targets. Second, we use image data collected onboard ground and aerial vehicles to demonstrate that our representation is amenable to an effective out-of-distribution (OOD) rejection technique that significantly improves the robustness of rotation estimates to unseen environmental effects and corrupted input images, without requiring the use of an explicit likelihood loss, stochastic sampling, or an auxiliary classifier. This capability is key for safety-critical applications where detecting novel inputs can prevent catastrophic failure of learned models.

\end{abstract}

\IEEEpeerreviewmaketitle

\section{Introduction}
\urlstyle{tt}

Rotation estimation constitutes one of the core challenges in robotic state estimation.\blfootnote{This material is based upon work that was supported in part by the Army Research Laboratory under Cooperative Agreement Number W911NF-17-2-0181. Their support is gratefully acknowledged.} Given the broad interest in applying deep learning to state estimation tasks involving rotations \cite{wang_deep_2019,tang_ba-net_2019,ranftl_deep_2018, peretroukhin_dpc-net_2018, peretroukhin_deep_2019, brachmann_neural-guided_2019,yi_learning_2018,clark_learning_2018, sattler_understanding_2019}, we consider the suitability of different rotation representations in this domain. The question of which rotation parameterization to use for estimation and control problems has a long history in aerospace engineering and robotics \citep{diebel_representing_2006}. In learning, unit quaternions (also known as Euler parameters) are a popular choice for their numerical efficiency, lack of singularities, and simple algebraic and geometric structure. Nevertheless, a standard unit quaternion parameterization does not satisfy an important continuity property that is essential for learning  arbitrary rotation targets, as recently detailed in \citep{zhou_continuity_2019}. To address this deficiency, the authors of \citep{zhou_continuity_2019} derived two alternative rotation representations that satisfy this property and lead to better network performance. Both of these representations, however, are \textit{point} representations, and do not quantify network uncertainty---an important capability in safety-critical applications. 

In this work, we introduce a novel representation of $\LieGroupSO{3}$ based on a symmetric matrix that combines these two important properties. Namely, it
\begin{enumerate}
	\item admits a smooth global section from $\LieGroupSO{3}$ to the representation space (satisfying the continuity property identified by the authors of \citep{zhou_continuity_2019});
	\item defines a Bingham distribution over unit quaternions; and 
	\item is amenable to a novel out-of-distribution (OOD) detection method without any additional stochastic sampling, or auxiliary classifiers.
\end{enumerate}
\Cref{fig:front_page} visually summarizes our approach. Our experiments use synthetic and real datasets to highlight the key advantages of our approach. We provide open source Python code\footnote{Code available at \url{https://github.com/utiasSTARS/bingham-rotation-learning}.} of our method and experiments. Finally, we note that our representation can be implemented in only a few lines of code in modern deep learning libraries such as PyTorch, and has marginal computational overhead for typical learning pipelines.

\begin{figure}
\centering 
\includegraphics[width=\columnwidth]{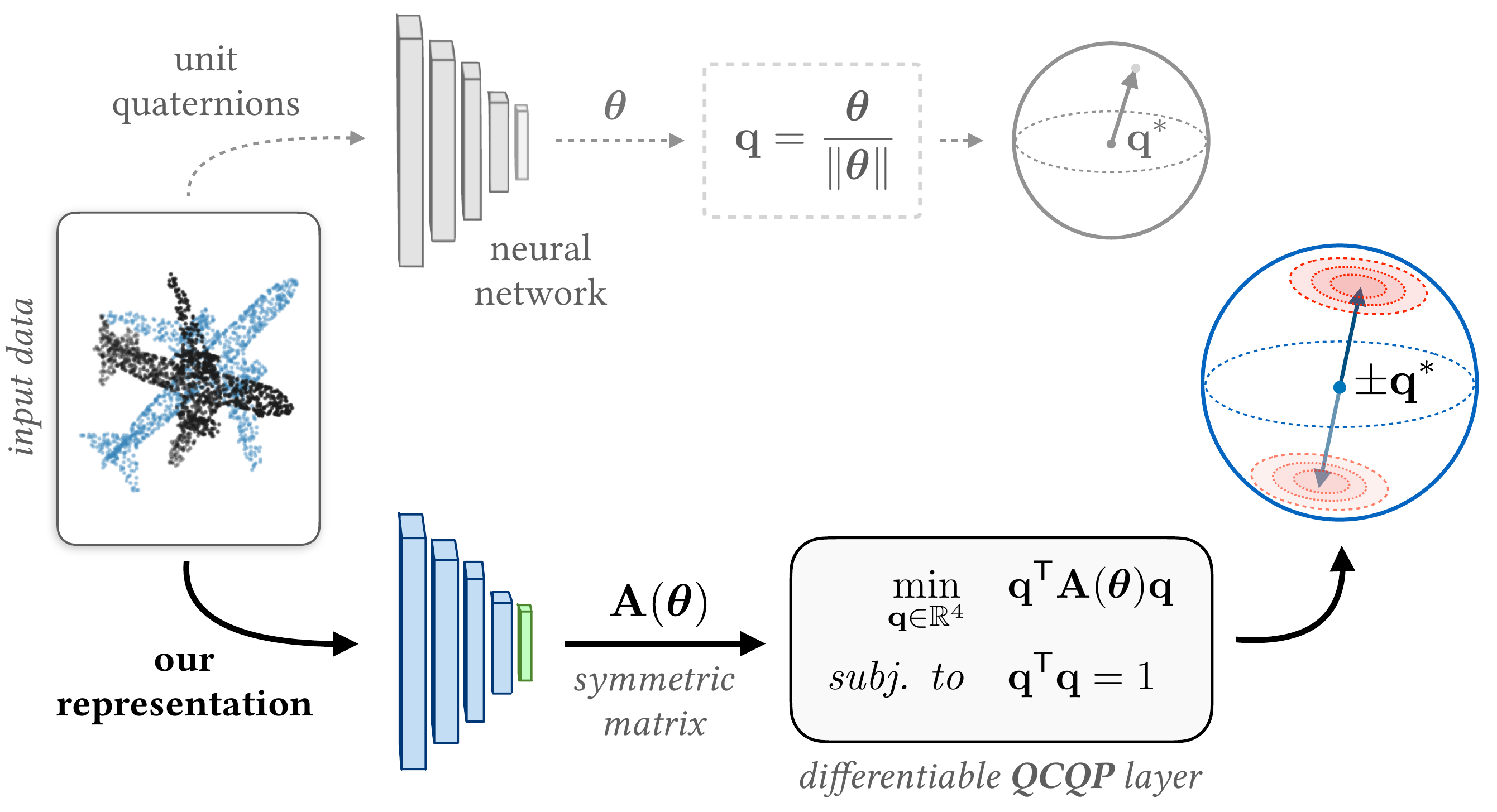}
\caption{We represent rotations through a symmetric matrix, $\Matrix{A}$, that defines a Bingham distribution over unit quaternions. To apply this representation to deep rotation regression, we present a differentiable layer parameterized by $\Matrix{A}$ and show how we can extract a notion of uncertainty from the spectrum of $\Matrix{A}$.}
\label{fig:front_page}
\end{figure}

\section{Related Work} \label{sec:related_work}

Estimating rotations has a long and rich history in computer vision and robotics \citep{scaramuzza_visual_2011,wahba_problem_1965,horn_closed-form_1987}. An in-depth survey of rotation averaging problem formulations and solution methods that deal directly with multiple rotation measurements is presented in \citep{hartley_rotation_2013}. In this section, we briefly survey techniques that estimate rotations from raw sensor data, with a particular focus on prior work that incorporates machine learning into the rotation estimation pipeline. We also review recent work on differentiable optimization problems and convex relaxation-based solutions to rotation estimation problems that inspired our work. 

\subsection{Rotation Parameterization}
In robotics, it is common to parameterize rotation states as elements of the matrix Lie group $\LieGroupSO{3}$ \cite{sola_micro_2018, barfoot_state_2017}. This approach facilitates the application of Gauss-Newton-based optimization and a local uncertainty quantification through small perturbations defined in a tangent space about an operating point in the group. In other state estimation contexts, applications may eschew full $3\times3$ orthogonal matrices with positive determinant (i.e., elements of $\LieGroupSO{3}$) in favour of lower-dimensional representations with desirable properties \citep{diebel_representing_2006}. For example, Euler angles \citep{sra_directional_2016} are particularly well-suited to analyzing small perturbations in the steady-state flight of conventional aircraft because reaching a singularity is practically impossible \citep{etkin_dynamics_1972}. In contrast, spacecraft control often requires large-angle maneuvers for which singularity-free unit quaternions are a popular choice \citep{wie_quaternion_1985}.

\subsection{Learning-based Rotation Estimation}
\label{sec:learning-based-est}
Much recent robotics literature has focused on improving classical pose estimation with learned models. Learning can help improve outlier classification \citep{yi_learning_2018}, guide random sample consensus \citep{brachmann_neural-guided_2019}, and fine-tune robust losses \citep{ranftl_deep_2018}. Further, fusing learned models of rotation with classical pipelines has been shown to improve accuracy and robustness of egomotion estimates \citep{peretroukhin_dpc-net_2018,peretroukhin_deep_2019}.

In many vision contexts, differentiable solvers have been proposed to incorporate learning into bundle adjustment \citep{tang_ba-net_2019}, monocular stereo \citep{clark_learning_2018}, point cloud registration \cite{wang_deep_2019}, and fundamental matrix estimation \citep{ranftl_deep_2018}. All of these methods rely on either differentiating a singular value decomposition \citep{ranftl_deep_2018, wang_deep_2019}, or `unrolling' local iterative solvers for a fixed number of iterations \citep{tang_ba-net_2019,clark_learning_2018}. Furthermore, adding interpretable outputs to a learned pipeline has been shown to improve generalization \citep{zhou_does_2019} and the paradigm of differentiable pipelines has been suggested to tackle an array of different robotics tasks \citep{karkus_differentiable_2019}.

The effectiveness of learning with various $\LieGroupSO{3}$ representations is explicitly addressed in \citep{zhou_continuity_2019}. 
 Given a \textit{representation} of $\LieGroupSO{3}$, by which we mean a surjective mapping $f : \mathcal{X} \rightarrow \LieGroupSO{3}$, the authors of \citep{zhou_continuity_2019} identified the existence of a continuous right-inverse of $f$, $g : \LieGroupSO{3} \rightarrow \mathcal{X}$, as important for learning.  Intuitively, the existence of such a $g$ ensures that the training signal remains continuous for regression tasks, reducing errors on unseen inputs. Similarly, an empirical comparison of $\LieGroupSE{3}$ representations for learning complex forward kinematics is conducted in \citep{grassmann_merits_2019}. Although full $\LieGroupSE{3}$ pose estimation is important in most robotics applications, we limit our analysis to $\LieGroupSO{3}$ representations as most pose estimation tasks can be decoupled into rotation and translation components, and the rotation component of $\LieGroupSE{3}$ constitutes the main challenge.

\subsection{Learning Rotations with Uncertainty}
Common ways to extract uncertainty from neural networks include approximate variational inference through Monte Carlo dropout \cite{gal_uncertainty_2016} and bootstrap-based uncertainty through an ensemble of models \cite{lakshminarayanan_simple_2017} or with multiple network \textit{heads} \cite{osband_deep_2016}. In prior work \cite{peretroukhin_deep_2019}, we have proposed a mechanism that extends these methods to $\LieGroupSO{3}$ targets through differentiable quaternion averaging and a local notion of uncertainty in the tangent space of the mean. 

Additionally, learned methods can be equipped with novelty detection mechanisms (often referred to as out-of-distribution or OOD detection) to help account for epistemic uncertainty. An autoencoder-based approach to OOD detection  was used on a visual navigation task in \citep{richter_safe_2017} to ensure that a safe control policy was used in novel environments. A similar approach was applied in \citep{amini_variational_2018}, where a single variational autoencoder was used for novelty detection and control policy learning. See \citep{meinke_neural_2019} for a recent summary of OOD methods commonly applied to classification tasks. 

 Finally, unit quaternions are also important in the broad body of work related to learning directional statistics \cite{sra_directional_2016} that enable global notions of uncertainty through densities like the Bingham distribution, which are especially useful in modelling large-error rotational distributions \cite{glover_tracking_2013, glitschenski_deep_2020}. Since we demonstrate that our representation parameterizes a Bingham belief, it is perhaps most similar to a recently published approach that uses a Bingham likelihood to learn global uncertainty over rotations \cite{glitschenski_deep_2020}. Our work differs from this approach in several important aspects: (1) our formulation is more parsimonious; it requires only a single symmetric matrix with 10 parameters to encode both the mode and uncertainty of the Bingham belief, (2) we present analytic gradients for our approach by considering a generalized QCQP-based optimization over rotations, (3) we prove that our representation admits a smooth right-inverse, and (4) we demonstrate that we can extract useful notions of uncertainty from our parameterization without using an explicit Bingham likelihood during training, avoiding the complex computation of the Bingham distribution's normalization constant.

\section{Symmetric Matrices and $\LieGroupSO{3}$}

Our rotation representation is defined using the set of real symmetric $4 \times 4$ matrices with a simple (i.e., non-repeated) minimum eigenvalue:
\begin{equation}
\label{eq:our_rep}
\Matrix{A} \in \SymmetricMatrices{4} : \lambda_1(\Matrix{A}) \neq \lambda_2(\Matrix{A}),
\end{equation}
where $\lambda_i$ are the eigenvalues of $\Matrix{A}$ arranged such that $\lambda_1 \leq \lambda_2 \leq \lambda_3 \leq \lambda_4$, and $\SymmetricMatrices{n}\triangleq \{\Matrix{A} \in \Real^{n \times n} : \Matrix{A} = \Matrix{A}^\T\}$. 
Each such matrix can be mapped to a unique rotation through a differentiable quadratically-constrained quadratic program (QCQP) defined in \Cref{fig:layer} and by \Cref{prob:quaternion_qcqp}. This representation has several advantages in the context of learned models. In this section, we will (1) show how the matrix $\Matrix{A}$ arises as the data matrix of a parametric QCQP and present its analytic derivative, (2) show that our representation is \textit{continuous} (in the sense of \citep{zhou_continuity_2019}), (3) relate it to the Bingham distribution over unit quaternions, and (4) discuss an intimate connection to rotation averaging.

\subsection{Rotations as Solutions to QCQPs} \label{sec:rotation_qcqps}
Many optimizations that involve rotations can be written as the constrained quadratic loss
\begin{align}
\min_{\Vector{x} \in \Real^{n}} &\quad \Vector{x}^\T \Matrix{A} \Vector{x}  \\
\text{\textit{subj. to}} &\quad \Vector{x} \in \mathcal{C} \notag,
\end{align}
where $\Vector{x}$ parameterizes a rotation, $\mathcal{C}$ defines a set of appropriate constraints, and $\Matrix{A}$ is a \textit{data matrix} that encodes error primitives, data association, and uncertainty. For example, consider the Wahba problem (WP) \citep{wahba_problem_1965}:

\begin{problem}[WP with unit quaternions]
\label{prob:wahba_quat}
Given a set of associated vector measurements $\{\Vector{u}_i, \Vector{v}_i\}_{i=1}^N \subset \Real^3$, find $\Quaternion \in S^3$ that solves
\begin{align}
\min_{\Quaternion \in S^3}\sum_{i=1}^N \frac{1}{\sigma_i^2}\Norm{\hat{\Vector{v}}_i - \Quaternion \QuatProd \hat{\Vector{u}}_i \QuatProd \Quaternion^{-1}}^2.
\end{align}
where $\hat{\Vector{p}} \Defined \bbm \Vector{p}^\T & 0 \ebm^\T$ is the homogenization of vector $\Vector{p}$ and $\QuatProd$ refers to the quaternion product.
\end{problem}
\noindent We can convert \Cref{prob:wahba_quat} to the following Quadratically-Constrained Quadratic Program (QCQP): 

\begin{problem}[WP as a QCQP]
\label{prob:wahba_qcqp}
Find $\Quaternion \in \Real^{4}$ that solves
\begin{align}
\min_{\Quaternion \in \Real^{4}} &\quad \Quaternion^\T \Matrix{A} \Quaternion  \\
\text{\textit{subj. to}} &\quad \Quaternion^\T\Quaternion = 1 \notag,
\end{align}
where the data matrix, $\Matrix{A} = \sum_{i=1}^N \Matrix{A}_i \in \SymmetricMatrices{4}$, is the sum of $N$ terms each given by 
\begin{align} \label{eq:wahba_cost_matrix}
\Matrix{A}_i = \frac{1}{\sigma_i^2} \left( (\Norm{\Vector{u}_i}^2 + \Norm{\Vector{v}_i}^2 ) \Identity + 2 \QuatMat_\ell(\hat{\Vector{v}}_i)\QuatMat_r(\hat{\Vector{u}}_i) \right), 
\end{align}
where $\QuatMat_\ell$ and $\QuatMat_r$ are left and right quaternion product matrices, respectively (cf. \cite{yang_quaternion-based_2019}).
\end{problem}
\noindent Constructing such an $\Matrix{A}$ for input data that does not contain vector correspondences constitutes the primary challenge of rotation estimation. Consequently, we consider applying the tools of high capacity data-driven learning to the task of predicting $\Matrix{A}$ for a given input. To this end, we generalize \Cref{prob:wahba_qcqp} and consider a parametric symmetric matrix $\Matrix{A}(\Vector{\theta})$:
\begin{problem}[Parametric Quaternion QCQP]
\label{prob:quaternion_qcqp}	
\begin{align}
\min_{\Quaternion \in \Real^{4}} &\quad \Quaternion^\T \Matrix{A}(\Vector{\theta})\Quaternion \\ 
\text{\textit{subj. to}} &\quad \Quaternion^\T\Quaternion = 1 \notag,
\end{align}
where $\Matrix{A}(\Vector{\theta}) \in \SymmetricMatrices{4}$ defines a quadratic cost function parameterized by $\Vector{\theta} \in \Real^{10}$.
\end{problem}
\subsubsection{Solving \Cref{prob:quaternion_qcqp}}
\Cref{prob:quaternion_qcqp} is minimized by a pair of antipodal unit quaternions, $\pm \Quaternion^*$, lying in the one-dimensional minimum-eigenspace of $\Matrix{A}(\Vector{\theta})$ \citep{horn_closed-form_1987}. Let $\SymmetricMatrices{4}_\lambda$ be the subset of $\SymmetricMatrices{4}$ with simple minimal eigenvalue $\lambda_1$: 
\begin{equation}
	\SymmetricMatrices{4}_\lambda \triangleq \{\Matrix{A} \in \SymmetricMatrices{4} : \lambda_1(\Matrix{A}) \neq \lambda_2(\Matrix{A}) \}
\end{equation} 
For any $\Matrix{A}(\Vector{\theta}) \in \SymmetricMatrices{4}_\lambda$, \Cref{prob:quaternion_qcqp} admits the solution $\pm \Quaternion^*$. Since $\LieGroupSO{3}$ is the quotient space of the unit quaternions obtained by identifying antipodal points, this represents a single rotation solution $\Rotation^* \in \LieGroupSO{3}$. Eigendecomposition of a real symmetric matrix can be implemented efficiently in most deep learning frameworks (e.g., we use the \texttt{symeig} function in  PyTorch). In practice, we encode $\Matrix{A}$ as 
\begin{equation}
\label{eq:A_theta_map}
\Matrix{A}(\Vector{\theta}) = \bbm \theta_1 & \theta_2 & \theta_3 & \theta_4 \\  \theta_2 & \theta_5 & \theta_6 & \theta_7 \\  \theta_3 & \theta_6 & \theta_8 & \theta_9 \\  \theta_4 & \theta_7 & \theta_9 & \theta_{10} \ebm,	
\end{equation}
 with no restrictions on $\Vector{\theta}$ to ensure that $\lambda_1(\Matrix{A}) \neq \lambda_2(\Matrix{A})$. We find that this does not impede training, and note that the complement of $\SymmetricMatrices{4}_\lambda$ is a set of measure zero in $\SymmetricMatrices{4}$.
\subsubsection{Differentiating \Cref{prob:quaternion_qcqp} wrt $\Matrix{A}$}
The derivative of $\Quaternion^*$ with respect to $\Vector{\theta}$ is guaranteed to exist if $\lambda_1$ is simple \citep{magnus_differentiating_1985}. Indeed, one can use the implicit function theorem to show that
\begin{equation}
	\PartialDerivative{\Quaternion^*}{\text{vec}(\Matrix{A})} = \Quaternion^* \KronProd \left( \lambda_1 \Identity -\Matrix{A} \right)^\dagger,
\end{equation}
where $(\cdot)^\dagger$ denotes the Moore-Penrose pseudo-inverse, $\KronProd$ is the Kronecker product, and $\Identity$ refers to the identity matrix. This gradient is implemented within the \texttt{symeig} function in PyTorch and can be efficiently implemented in any framework that allows for batch linear system solves.
\begin{figure}[t]
\centering
\includegraphics[width=0.75\columnwidth]{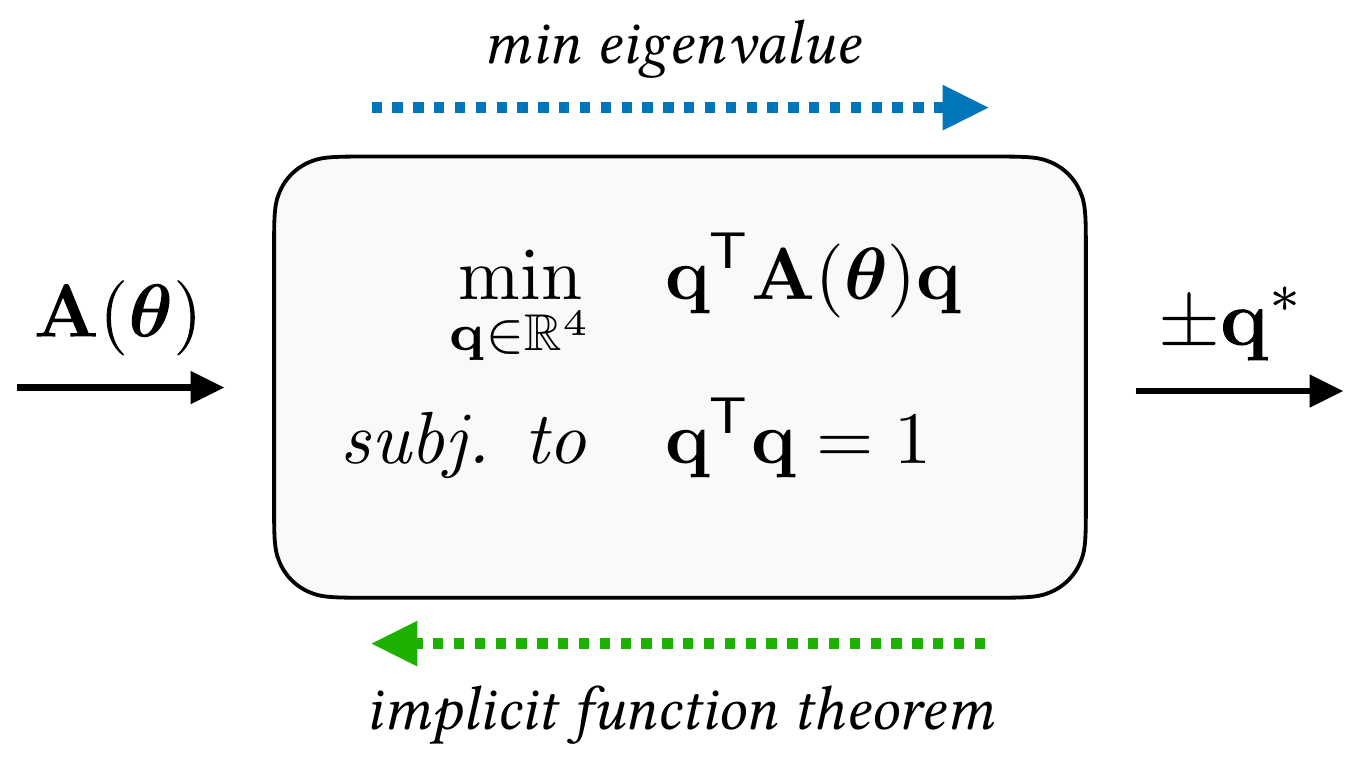}
\caption{A differentiable QCQP layer representing \Cref{prob:quaternion_qcqp}. The layer takes as input a symmetric matrix $\Matrix{A}$ defined by the parameters $\Vector{\theta}$. The solution is given by the minimum-eigenspace of $\Matrix{A}$ and the implicit function theorem can be used to derive an analytic gradient.}	
\label{fig:layer}
\end{figure}

\subsection{A Smooth Global Section of $\LieGroupSO{3}$} \label{sec:representation}
\noindent Consider the surjective map\footnote{Surjectivity follows from the fact that $f$ admits a global section \Cref{prop:continuous_inverse}.}  ${f: \SymmetricMatrices{4}_\lambda \rightarrow \LieGroupSO{3}}$:
\begin{equation}
\label{eq:f_def}
	f: \Matrix{A} \mapsto \Rotation^*,
\end{equation}
where $\Rotation^*$ is the rotation matrix corresponding to 
\begin{equation} \label{eq:q_star}
\pm\Quaternion^* = \ArgMin{\Quaternion \in S^3} \Quaternion^\T\Matrix{A}\Quaternion.
\end{equation}
 As noted in \Cref{sec:learning-based-est}, the authors of \cite{zhou_continuity_2019} identified the existence of a continuous right-inverse, or section, of $f$ as important for learning. The authors further used topological arguments to demonstrate that a continuous representation is only possible if the dimension of the embedding space is greater than four. In the case of four-dimensional unit quaternions, this discontinuity manifests itself at 180$^\circ$ rotations. For our ten-dimensional representation, we present a proof that one such (non-unique) continuous mapping, $g: \LieGroupSO{3} \rightarrow \SymmetricMatrices{4}_\lambda$, exists and is indeed smooth. 
 
\begin{theorem}[Smooth Global Section, $\LieGroupSO{3} \rightarrow \SymmetricMatrices{4}_\lambda$]
\label{prop:continuous_inverse}
	Consider the surjective map $f: \SymmetricMatrices{4}_\lambda \rightarrow \LieGroupSO{3}$ such that $f(\Matrix{A})$ returns the rotation matrix defined by the two antipodal unit quaternions $\pm \Quaternion^*$ that minimize \Cref{prob:quaternion_qcqp}. There exists a smooth and global mapping, or section, $g: \LieGroupSO{3} \rightarrow \SymmetricMatrices{4}_\lambda$ such that $f(g(\Rotation)) = \Rotation$. 
\end{theorem}
\begin{proof}
 Recall that the mapping $\Rotation(\cdot) : S^3 \rightarrow \LieGroupSO{3}$ from unit quaternions to rotations is continuous, surjective, and identifies antipodal unit-quaternions (i.e., sends them to the same rotation); this shows that $\LieGroupSO{3} \cong \RealProjective{3}$ ($\LieGroupSO{3}$ is \textit{diffeomorphic} to $\RealProjective{3}$) as smooth manifolds.  Therefore, it suffices to show that the global section $g: \RealProjective{3} \rightarrow \SymmetricMatrices{4}_\lambda$ exists.
 Let $[\Quaternion]$ be an arbitrary element of $\RealProjective{3}$ and define
 \begin{equation}
 	g([\Quaternion]) \Defined \Identity -\Quaternion\Quaternion^\T,
 \end{equation}
where $\Quaternion$ is one of the two representatives ($\pm \Quaternion$) of $[\Quaternion]$ in $S^3$.  Note that $g(\cdot)$ is well-defined over arbitrary elements of $\RealProjective{3}$, since selecting either representative leads to the same output (i.e., $g(-\Quaternion) = g(\Quaternion)$).  By construction, $g([\Quaternion])$ is the orthogonal projection operator onto the 3-dimensional orthogonal complement of $\text{span}(\Quaternion) = \text{span}(-\Quaternion)$ in $\Real^4$.  Therefore, $\lambda\{ g([\Quaternion]) \} = \lambda\{\Identity - \Quaternion\Quaternion^\T\} = \{0, 1, 1, 1\}$. It follows that $g([\Quaternion])$ defines a symmetric matrix with a simple minimum-eigenvalue (i.e., $g([\Quaternion]) \in \SymmetricMatrices{4}_\lambda$) and the eigenspace associated with the minimum eigenvalue of 0 is precisely $\text{span}(\Quaternion) = \text{span}(-\Quaternion)$.  This in turn implies that:
\begin{equation}
\pm\Quaternion =  \ArgMin{\Quaternion \in S^3} \Quaternion^\T g\left([\Quaternion]\right) \Quaternion,
\end{equation}
and therefore $f\left(g\left([\Quaternion]\right)\right) = [\Quaternion]$ so that $g(\cdot)$ is a global section of the surjective map $f(\cdot)$.  Furthermore, we can see by inspection that this global section is smooth (i.e., continuous and differentiable) since we can always represent $g(\cdot)$ locally using one of the diffeomorphic preimages of $\RealProjective{3}$ in $S^3$ as the smooth function $g_0(\Quaternion) = \Identity - \Quaternion\Quaternion^\T$. 
\end{proof}

\subsection{$\Matrix{A}$ and the Bingham Distribution}
\label{sec:bingham}
We can further show that our representation space, $\Matrix{A}(\Vector{\theta}) \in \SymmetricMatrices{4}_\lambda$, defines a Bingham distribution over unit quaternions. Consequently, we may regard $\Matrix{A}$ as encoding a \textit{belief} over rotations which facilitates the training of rotation models with uncertainty. The Bingham distribution is an antipodally symmetric distribution that is derived from a zero-mean Gaussian in $\Real^{d+1}$ conditioned to lie on the unit hypersphere, $S^d$ \cite{bingham_antipodally_1974}. For unit quaternions ($d=3$),  the probability density function of $\Vector{x} \sim \BinghamDistribution{\Matrix{D}}{\Matrix{\Lambda}}$ is 
\begin{align}
 p(\Vector{x}; \Matrix{D}, \Matrix{\Lambda}) &= \frac{1}{N(\Matrix{\Lambda})} \exp{\left(\sum_{i=1}^3 \dc_i (\Vector{d}_i^\T \Vector{x})^2\right)} \\
 &= \frac{1}{N(\Matrix{\Lambda})} \exp{\left( \Vector{x}^\T \Matrix{D} \Matrix{\Lambda} \Matrix{D}^\T \Vector{x}\right)},
\end{align}
where $\Vector{x} \in S^3$, $N(\Matrix{\Lambda})$ is a normalization constant, and $\Matrix{D} \in O(4)$ is an orthogonal matrix formed from the three orthogonal unit column vectors $\Vector{d}_i$ and a fourth mutually orthogonal unit vector, $\Vector{d}_4$. The matrix of dispersion coefficients, $\Matrix{\Lambda}$, is given by $\diag{\dc_1, \dc_2, \dc_3, 0}$ with $\dc_1 \leq \dc_2 \leq \dc_3 \leq 0$ (note that these dispersion coefficients are eigenvalues of the matrix $\Matrix{D} \Matrix{\Lambda} \Matrix{D}^\T$).  Each $\dc_i$ controls the spread of the probability mass along the direction given by $\Vector{d}_i$ (a small magnitude $\dc$ implying a large spread and vice-versa). The mode of the distribution is given by $\Vector{d}_4$. 

Crucially, $\BinghamDistribution{\Matrix{D}}{\Matrix{\Lambda}} = \BinghamDistribution{\Matrix{D}}{\Matrix{\Lambda} + c \Identity }$ for all $c \in \Real$ \citep{darling_uncertainty_2016}. Thus, a careful diagonalization of our representation, $\Matrix{A} \in \SymmetricMatrices{4}_\lambda$, fully describes $\BinghamDistribution{\Matrix{D}}{\Matrix{\Lambda}}$. Namely, to ensure that the mode of the density is given by the solution of \Cref{prob:quaternion_qcqp}, we set $\Matrix{D} \Matrix{\Lambda} \Matrix{D}^\T = -\Matrix{A}$ since the smallest eigenvalue of $\Matrix{A}$ is the largest eigenvalue of $-\Matrix{A}$.

To recover the dispersion coefficients, $\dc_i$, we evaluate the non-zero eigenvalues of $-\Matrix{A} + {\eigval_1} \Identity$ (defining the equivalent density $\BinghamDistribution{\Matrix{D}}{\Matrix{\Lambda} + \lambda_1 \Identity }$) where $\eigval_i$ are the eigenvalues of $\Matrix{A}$ in ascending order as in \Cref{eq:our_rep}. Then, $\{\dc_1, \dc_2, \dc_3 \} = \{-\eigval_4 + \eigval_1, -\eigval_3 + \eigval_1, -\eigval_2 + \eigval_1\}$.  This establishes a relation between the eigenvalue \textit{gaps} of $\Matrix{A}$ and the dispersion of the Bingham density defined by $\Matrix{A}$. 
 
\section{Using $\Matrix{A}$ with Deep Learning}
\label{sec:A_deep_learning}
We consider applying our formulation to the learning task of fitting a set of parameters, $\Vector{\pi}$, such that the deep rotation regressor $\Rotation = \mathtt{NN}(\Vector{x}; \Vector{\pi})$ minimizes a training loss $\mathcal{L}$ while generalizing to unseen data (as depicted in \Cref{fig:front_page}).
\subsection{Self-Supervised Learning}
In many self-supervised learning applications, one requires a rotation matrix to transform one set of data onto another. Our representation admits a differentiable transformation into a rotation matrix though $\Rotation^*=\Rotation(\pm \Quaternion^*)$, where $\Rotation(\Quaternion)$ is the unit quaternion to rotation matrix projection (e.g., Eq. 4 in \cite{zhou_continuity_2019}). Since our layer solves for a pair of antipodal unit quaternions ($\pm \Quaternion^* \in \RealProjective{3} \cong \LieGroupSO{3}$), and therefore admits a smooth global section, we can avoid the discontinuity identified in \cite{zhou_continuity_2019} during back-propagation. In this work, however, we limit our attention to the supervised rotation regression problem to directly compare with the continuous representation of \cite{zhou_continuity_2019}.

\subsection{Supervised Learning: Rotation Loss Functions}
For supervised learning over rotations, there are a number of possible choices for loss functions that are defined over $\LieGroupSO{3}$. A survey of different bi-invariant metrics which are suitable for this task is presented in \citep{hartley_rotation_2013}. For example, four possible loss functions include,
\begin{align}
	\mathcal{L}_\text{quat}\left(\Quaternion, \Quaternion_{\text{gt}}\right) &= \QuatDist{\Quaternion}{\Quaternion_{\text{gt}}}^2,\\
	\mathcal{L}_\text{chord}\left(\Rotation, \Rotation_{\text{gt}}\right) &= \ChordDist{\Rotation}{\Rotation_{\text{gt}}}^2,\\
		\mathcal{L}_\text{ang}\left(\Rotation, \Rotation_{\text{gt}}\right) &= \AngDist{\Rotation}{\Rotation_{\text{gt}}}^2,\\
	\mathcal{L}_\textsc{Bingham}\left(\Matrix{D},  \Matrix{\Lambda}, \Quaternion_{\text{gt}}\right) &= \Quaternion_{\text{gt}}^\T \Matrix{D} \Matrix{\Lambda} \Matrix{D}^\T \Quaternion_{\text{gt}} +  N(\Matrix{\Lambda}),
\end{align}
where\footnote{Note that it is possible to relate all three metrics without converting between representations--e.g., $d_{\text{chord}}^2\left(\Rotation(\Quaternion),\Rotation(\Quaternion_{\text{gt}})\right) = 2d_{\text{quat}}^2(4 - d_{\text{quat}}^2)$.}
\begin{align}
\QuatDist{\Quaternion}{\Quaternion_{\text{gt}}} &= \min\left(\Norm{\Quaternion_{\text{gt}} - \Quaternion}_2, \Norm{\Quaternion_{\text{gt}} + \Quaternion}_2\right), \\
\ChordDist{\Rotation}{\Rotation_{\text{gt}}} &= \Norm{\Rotation_{\text{gt}} - \Rotation}_{\text{F}}, \\
\AngDist{\Rotation}{\Rotation_{\text{gt}}} &= \Norm{\MatLog{\Rotation\Rotation_{\text{gt}}^\T}},
\end{align}
and $\MatLog{\cdot}$ is defined as in \cite{sola_micro_2018}. Since our formulation fully describes a Bingham density, it is possible to use the likelihood loss, $\mathcal{L}_\textsc{Bingham}$, to train a Bingham belief (in a similar manner to \cite{glitschenski_deep_2020} who use an alternate 19-parameter representation). However, the normalization constant $N(\Matrix{\Lambda})$ is a hypergeometric function that is non-trivial to compute. The authors of \cite{glitschenski_deep_2020} evaluate this constant using a fixed-basis non-linear approximation aided by a precomputed look-up table. In this work, we opt instead to compare to other point representations and leave a comparison of different belief representations to future work. Throughout our experiments, we use the chordal loss $\mathcal{L}_\text{chord}$ which can be applied to both rotation matrix and unit quaternion outputs.\footnote{The chordal distance has also been shown to be effective for initialization and optimization in SLAM \cite{carlone_initialization_2015,rosen_se-sync_2019}.} However, despite eschewing a likelihood loss, we can still extract a useful notion of uncertainty from deep neural network regression using the eigenvalues of our matrix $\Matrix{A}$. To see why, we present a final interpretation of our representation specifically catered to the structure of deep models.

\section{$\Matrix{A}$ and Rotation Averaging} 
\label{sec:rotation_averaging}
We take inspiration from \citep{sattler_understanding_2019} wherein neural-network-based pose regression is related to an interpolation over a set of base poses. We further elucidate this connection by specifically considering rotation regression and relating it to rotation averaging over a (learned) set of base rotations using the chordal distance. Since these base rotations must span the training data, we argue that $\Matrix{A}$ can represent a notion of epistemic uncertainty (i.e., distance to training samples) without an explicit likelihood loss. 

Consider that given $N$ rotation samples expressed as unit quaternions $\Quaternion_i$, the rotation average according to the chordal metric can be computed as \cite{hartley_rotation_2013}:
\begin{align}
	\Mean{\Quaternion}_{d_\text{chord}} = \ArgMin{\Quaternion \in S^3} \sum_i^N d_{\text{chord}} (\Quaternion, \Quaternion_i) = f\left(-\sum_i^N \Quaternion_i \Quaternion_i^\T\right),
\end{align}
where $f(\cdot)$ is defined by\footnote{The negative on $\Quaternion_i \Quaternion_i^\T$ is necessary since $f$ computes the eigenvector associated with $\lambda_{\text{min}}$ whereas the average requires $\lambda_{\text{max}}$.} \Cref{eq:f_def}. A weighted averaging version of this operation is discussed in \citep{markley_averaging_2007}. Next, consider a feed-forward neural network that uses our representation by regressing ten parameters, $\Vector{\theta} \in \Real^{10}$. If such a network has a final fully-connected layer prior to its output, we can separate it into two components: (1) the last layer parameterized by the weight matrix $\Matrix{W} \in \Real^{10\times N}$ and the bias vector $\Vector{b} \in \Real^{10}$, and (2) the rest of the network $\Vector{\gamma}(\cdot)$ which transforms the input $\Vector{x}$ (e.g., an image) into $N$ coefficients given by $\gamma_i$. 
The output of such a network is then given by 
\begin{align}\
\Quaternion^* = f(\Matrix {A}(\Vector{\theta}(\Vector{x})))
&= f\left(\Matrix{A}\left(\sum_i^N \Vector{w}_i\gamma_i(\Vector{x}) + \Vector{b}\right)\right) \\
&= f\left(\sum_i^N \Matrix{A}(\Vector{w}_i)\gamma_i(\Vector{x}) + \Matrix{A}(\Vector{b})\right)
\end{align}
where $\Vector{w}_i$ refers to the $i$th column of $\Matrix{W}$ and the second line follows from the linearity of the mapping defined in \Cref{eq:A_theta_map}. In this manner, we can view rotation regression with our representation as analogous to computing a weighted chordal average over a set of learned base orientations (parameterized by the symmetric matrices defined by the column vectors $\Vector{w}_i$ and $\Vector{b}$). During training the network tunes both the bases and the weight function $\Vector{\gamma}(\cdot)$.\footnote{We can make a similar argument for the unit quaternion representation since the normalization operation $f_n(\Vector{y}) = \Vector{y} \Norm{\Vector{y}}^{-1}$ corresponds to the mean $\Mean{\Quaternion}_{d_\text{quat}} = \ArgMin{\Quaternion \in S^3} \sum_i^N d_{\text{quat}} (\Vector{q}, \Vector{q}_i) = f_n\left(\sum_i^N \Quaternion_i\right)$.}

\subsection{Dispersion Thresholding (DT) as Epistemic Uncertainty}
\label{sec:dispersion_thresholding}
The positive semi-definite (PSD) matrix $\sum_i^N \Quaternion_i \Quaternion_i^\T$ is also called the \textit{inertia} matrix in the context of Bingham maximum likelihood estimation because the operation $f\left(-\sum_i^N \Quaternion_i \Quaternion_i^\T\right)$ can be used to compute the maximum likelihood estimate of the mode of a Bingham belief given $N$ samples \cite{glover_tracking_2013}.  Although our symmetric representation $\Matrix{A}$ is not necessarily PSD, we can perform the same canonicalization as in \Cref{sec:bingham}, and interpret  $\Matrix{\Lambda} = -\Matrix{A} + {\eigval_1} \Identity$ as the (negative) inertia matrix. Making the connection to Bingham beliefs, we then use 
\begin{align}
	\Trace{\Matrix{\Lambda}} = \sum_i^3 \dc_i = 3\eigval_1 - \eigval_2 - \eigval_3 - \eigval_4
\end{align}
as a measure of network uncertainty.\footnote{Note that under this interpretation the matrix $\Matrix{\Lambda}$ does not refer directly to the Bingham dispersion matrix parameterized by the inertia matrix, but the two can be related by inverting an implicit equation--see \cite{glover_tracking_2013}.} We find empirically that this works surprisingly well to measure epistemic uncertainty (i.e., model uncertainty) without any explicit penalty on $\Matrix{\Lambda}$ during training. To remove OOD samples, we compute a threshold on $\Trace{\Matrix{\Lambda}}$ based on quantiles over the training data (i.e., retaining the lowest $q$th quantile). We call this technique \textit{dispersion thresholding} or DT. Despite the intimate connection to both rotation averaging and Bingham densities, further work is required to elucidate why exactly this notion of uncertainty is present without the use of an uncertainty-aware loss like $\mathcal{L}_\textsc{Bingham}$. We leave a thorough investigation for future work, but note that this metric can be related to the norm $\Norm{\sum_i^N \Vector{w}_i\gamma_i(\Vector{x}) + \Vector{b}}$ which we find empirically to also work well as an uncertainty metric for other rotation representations. We conjecture that the `basis' rotations $\Vector{w}_i, \Vector{b}$ must have sufficient spread over the space to cover the training distribution. During test-time, OOD inputs, $\Vector{x}_\text{OOD}$, result in weights, $\gamma_i(\Vector{x}_\text{OOD})$, which are more likely to `average out' these bases due to the compact nature of $\LieGroupSO{3}$.  Conversely, samples closer to training data are more likely to be nearer to these learned bases and result in larger $\left|\Trace{\Matrix{\Lambda}}\right|$.

\section{Experiments} \label{sec:experiments}
We present the results of extensive synthetic and real experiments to validate the benefits of our representation. In each case, we compare three\footnote{We note that regressing $3\times3$ rotation matrices directly would also satisfy the continuity property of \cite{zhou_continuity_2019} but we chose not to include this as the \texttt{6D} representation fared better in the experiments of \citep{zhou_continuity_2019}.} representations of $\LieGroupSO{3}$: (1) unit quaternions (i.e., a normalized 4-vector, as outlined in \Cref{fig:front_page}), (2) the best-performing continuous six-dimensional representation, \texttt{6D}, from \citep{zhou_continuity_2019}, and (3) our own symmetric-matrix representation, $\Matrix{A}$. We report all rotational errors in degrees based on $\AngDist{\cdot}{\cdot}$. 

\subsection{Wahba Problem with Synthetic Data}
First, we simulated a dataset where we desired a rotation from two sets of unit vectors with known correspondence. We considered the generative model,
\begin{equation}
\Vector{v}_i = \hat{\Rotation} \Vector{u}_i + \Vector{\epsilon}_i~, ~~ \Vector{\epsilon}_i \sim \NormalDistribution{\Vector{0}}{\sigma^2 \Identity},
\end{equation}
where $\Vector{u}_i$ are sampled from the unit sphere. For each training and test example, we sampled $\hat{\Rotation}$ as $\MatExp{\Vector{\hat{\phi}}}$ (where we define the capitalized exponential map as \cite{sola_micro_2018}) with $\Vector{\hat{\phi}} = \phi \frac{\Vector{a}}{\Norm{\Vector{a}}}$ and $\Vector{a} \sim \NormalDistribution{\Vector{0}}{\Identity}$, $\phi \sim U[	0, \phi_{\text{max}})$, and set $\sigma = 0.01$. We compared the training and test errors for different learning rates in selected range using the \texttt{Adam} optimizer. Taking inspiration from \cite{zhou_continuity_2019}, we employed a dynamic training set and constructed each mini-batch from $100$ sampled rotations with $100$ noisy matches, ${\Vector{u}_i, \Vector{v}_i}$, each. We defined an epoch as five mini-batches. Our neural network structure mimicked the convolutional structure presented in \cite{zhou_continuity_2019} and we used $\mathcal{L}_\text{chord}$ to train all models. 

\begin{figure}[h!]	
\centering
\includegraphics[width=\columnwidth]{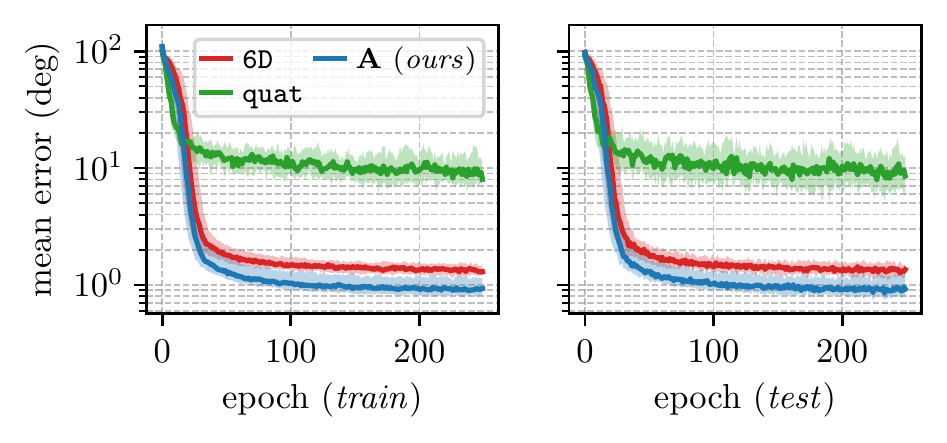}
\caption{Angular errors for 25 different trials each with learning rates sampled from the range $\{10^{-4}, 10^{-3}\}$ (log-uniform) and $\phi_{\text{max}} = 180^\circ$. We show $\{10, 50, 90\}^{\text{th}}$ percentiles at each epoch. }	
\label{fig:wabha_random_unit}
\end{figure}

\begin{figure}[h!]	
\centering
\includegraphics[width=\columnwidth]{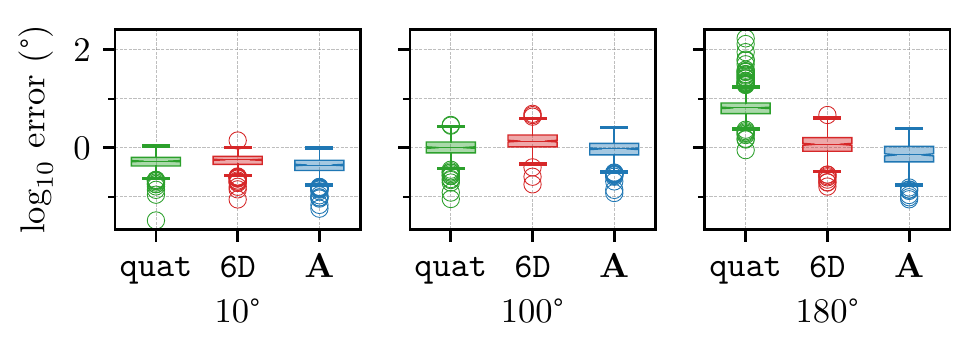}
\caption{Box and whiskers plots for three different settings of  $\phi_{\text{max}}$ for three rotation representations applied to synthetic data. The unit quaternion representation results in large errors as $\phi_{\text{max}} \rightarrow 180^\circ$.}	
\label{fig:synthetic_rotangle_box}
\vspace{-0.5cm}
\end{figure}

\Cref{fig:wabha_random_unit} displays the results of 25 experimental trials with different learning rates on synthetic data. For arbitrary rotation targets, both continuous representations outperform the discontinuous unit quaternions, corroborating the results of \citep{zhou_continuity_2019}. Moreover, our symmetric representation achieves the lowest errors across training and testing. \Cref{fig:synthetic_rotangle_box} depicts the performance of each representation on training data restricted to different maximum angles. As hypothesized in \cite{zhou_continuity_2019}, the discontinuity of the unit quaternion manifests itself on regression targets with angles of magnitude near 180 degrees.

\begin{figure}	
	\centering
	\begin{subfigure}[t]{\columnwidth}
\centering
\includegraphics[width=0.95\columnwidth]{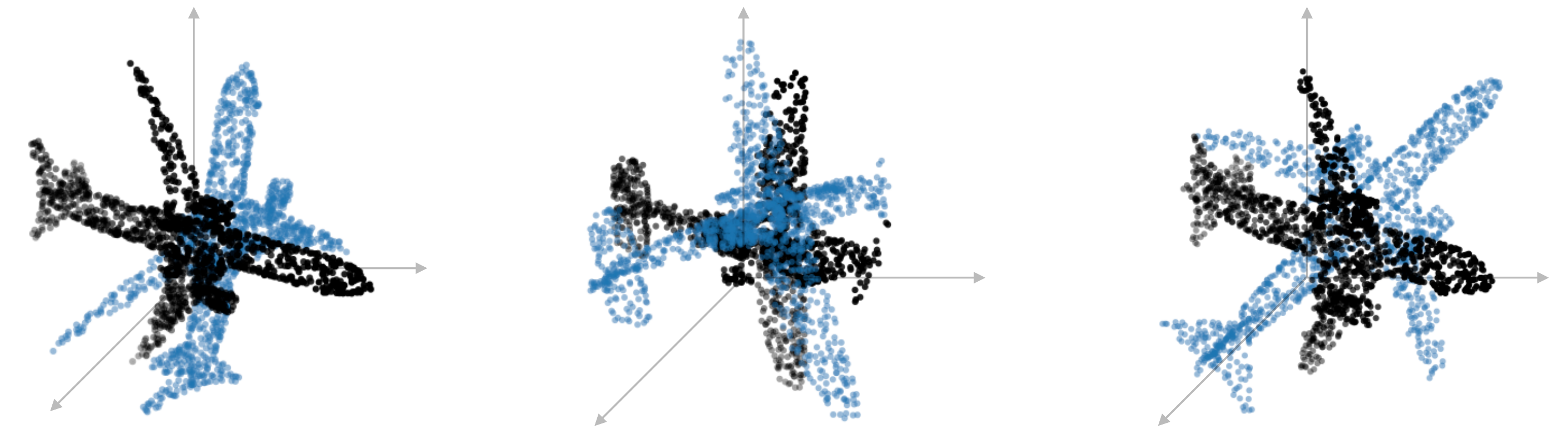}
\caption{We sample the \ShapeNet \texttt{airplane} category and randomly rotate point clouds to generate our training and test data.}	
\label{fig:shapenet_airplanes}
	\end{subfigure}
	\begin{subfigure}[t]{\columnwidth}
\centering
\includegraphics[width=\columnwidth]{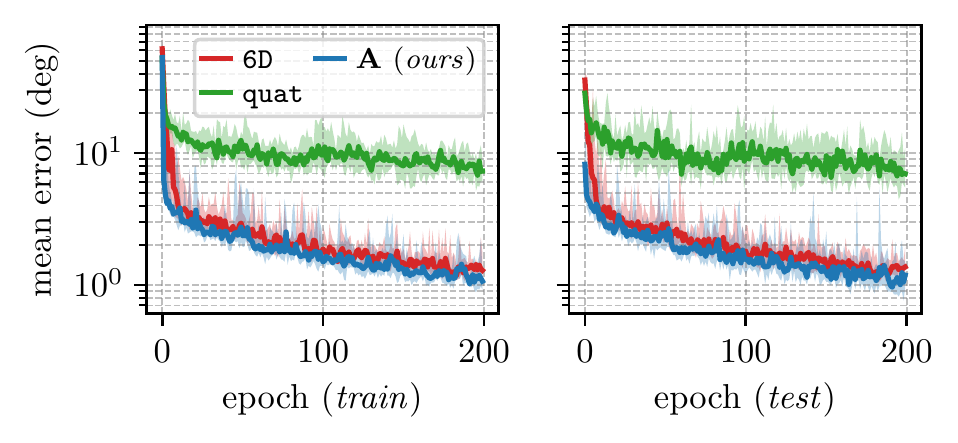}
\caption{Mean angular errors for 10 different \ShapeNet dataset trials each with learning rates sampled in the range $\{10^{-4}, 10^{-3}\}$ (log-uniform). We plot $\{10, 50, 90\}^{\text{th}}$ percentiles at each epoch.}
\label{fig:shapenet_learning_rates}
	\end{subfigure}
	\caption{A summary of our \ShapeNet experiments.}
	\vspace{-0.5cm}
\end{figure}

\subsection{Wahba Problem with \ShapeNet}

Next, we recreated the experiment from \cite{zhou_continuity_2019} on 2,290 airplane point clouds from \ShapeNet \cite{shapenet2015}, with 400 held-out point clouds. During each iteration of training we randomly selected a single point cloud and transformed it with 10 sampled rotation matrices (\Cref{fig:shapenet_airplanes}). At test time, we applied 100 random rotations to each of the 400 held-out point clouds. \Cref{fig:shapenet_learning_rates} compares the performance of our representation against that of unit quaternions and the \texttt{6D} representation, with results that are similar to the synthetic case in \Cref{fig:wabha_random_unit}.

\begin{figure}	
	\captionsetup[subfigure]{justification=centering}
	\centering
	\begin{subfigure}{0.52\columnwidth}
		\includegraphics[width=\linewidth]{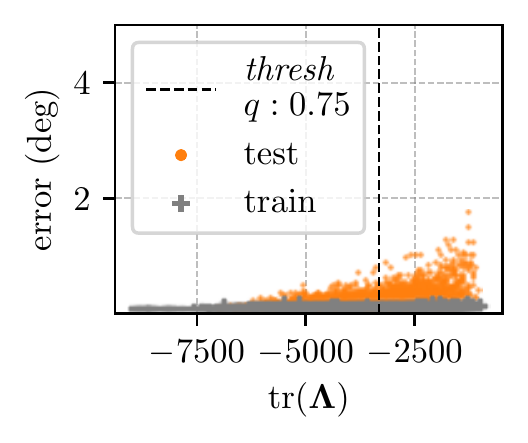}
		\caption{Sequence \texttt{02}.}
	\end{subfigure}\begin{subfigure}{0.47\columnwidth}
		\includegraphics[width=\linewidth]{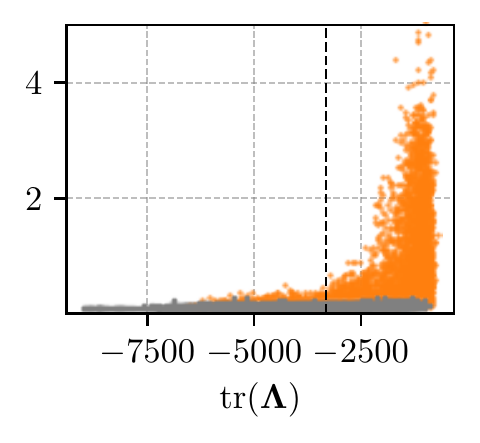}
		\caption{Sequence \texttt{02} (\textbf{corrupted}).}
	\end{subfigure} 
	\begin{subfigure}{\columnwidth}
		\includegraphics[width=\columnwidth]{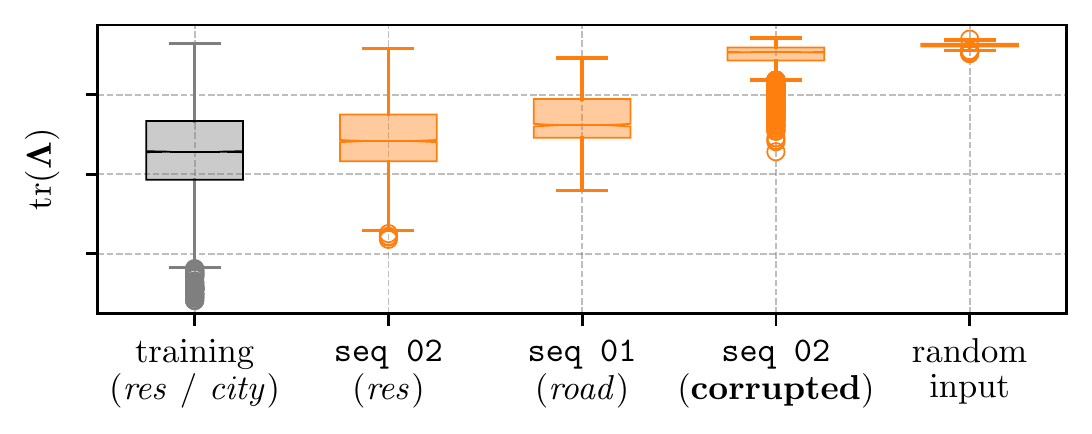}
	\caption{Different test sets.}
\label{fig:kitti_box_02}		
\end{subfigure}
	\caption{The dispersion thresholding metric, $\Trace{\Matrix{\Lambda}}$, plotted for data corresponding to test sequence \texttt{02} (refer to text for train/test split). Corruption is applied to the test data without altering the training data. DT thresholding leads to an effective OOD rejection scheme without the need to retrain the model.}
	\label{fig:kitti_egt_scatter}
	\vspace{-0.45cm}
\end{figure}

\begin{figure}	
	\centering
	\begin{subfigure}[t]{0.515\columnwidth}
		\centering
		\includegraphics[width=\linewidth]{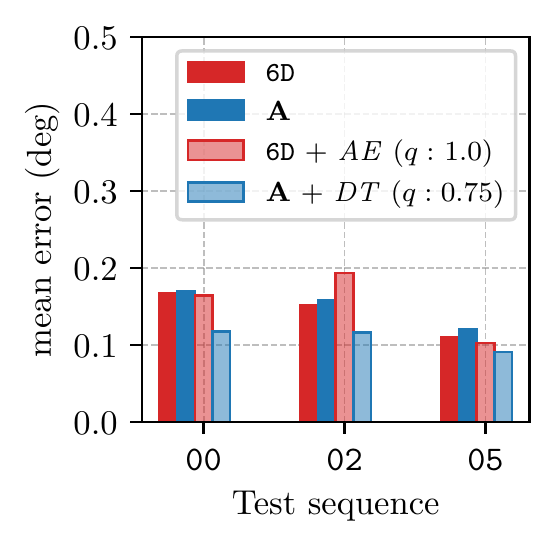}
		\caption{Uncorrupted images.}
		\label{fig:kitti_bar_errors_standarad}
	\end{subfigure}\begin{subfigure}[t]{0.47\columnwidth}
		\centering
		\includegraphics[width=\linewidth]{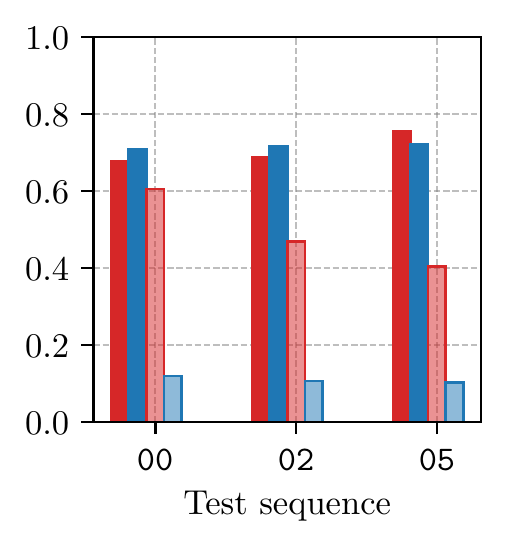}
		\caption{Corrupted images.}
		\label{fig:kitti_bar_errors_corrupted}
	\end{subfigure}
	\caption{Mean rotation errors for three different test sequences from the KITTI odometry dataset. Since these relative rotation targets are `small', we see no significant difference in baseline accuracy. However, our dispersion thresholding technique (\Cref{sec:dispersion_thresholding}) can significantly improve performance, and outperforms an alternative OOD method based on reconstruction error using an auto-encoder (AE). See \Cref{tab:kitti_relative_rotation_stats} for full statistics.}
	\vspace{-0.55cm}
	\label{fig:kitti_bar_errors}
\end{figure}

\begin{table*}[]
	\centering
	\caption{Relative rotation learning on the KITTI dataset with different representations of $\LieGroupSO{3}$. All training is done on uncorrupted data. We show that our OOD technique, DT, can dramatically improve rotation accuracy by rejecting inputs that are likely to lead to high error.}
	\begin{threeparttable}
	\begin{tabular}{clccccc}
		\toprule
		& \textbf{} & \multicolumn{2}{c}{\textbf{Normal Test}} & \multicolumn{3}{c}{\textbf{Corrupted Test (50\%)}} \\ \cmidrule{2-7} 
		\textbf{Sequence} & Model & Mean Error ($^{\circ}$) & Kept (\%) & Mean Error ($^{\circ}$) & Kept (\%) & Precision\tnote{3} (\%) \\ \midrule
		\multirow{5}{*}{\texttt{00} (4540 pairs)}  & \texttt{quat} & 0.16 & 100 & 0.74 & 100 & --- \\
		& \texttt{6D} \cite{zhou_continuity_2019} & 0.17 & 100 & 0.68 & 100 & ---\\
		& \texttt{6D} + auto-encoder (\textit{AE})\tnote{1} & 0.16 & 32.0 & 0.61 & 19.1 & 57.4\\
        & $\Matrix{A}$ \textit{(ours)} & 0.17 & 100 & 0.71 & 100 & ---\\
		& $\Matrix{A}$ + DT (\Cref{sec:dispersion_thresholding})\tnote{2} & \textbf{0.12} & 69.5 & \textbf{0.12} & 37.3 & \textbf{99.00} \\ \midrule
		
		\multirow{5}{*}{\texttt{02} (4660 pairs)}  & \texttt{quat} & 0.16 & 100 & 0.64 & 100 & --- \\
		& \texttt{6D} & 0.15 & 100 & 0.69 & 100 & ---\\
		& \texttt{6D} + AE\tnote{1} & 0.19 & 15.4 & 0.47 & 9.9 & 77.83\\
		& $\Matrix{A}$  & 0.16 & 100 & 0.72 & 100 & ---\\
		& $\Matrix{A}$ + DT\tnote{2} & \textbf{0.12} & 70.1 & \textbf{0.11} & 34.0 & \textbf{99.50} \\ \midrule
		\multirow{5}{*}{\texttt{05} (2760 pairs)}  & \texttt{quat} & 0.13 & 100 & 0.72 & 100 & --- \\
		& \texttt{6D} & 0.11 & 100 & 0.76 & 100 & ---\\
		& \texttt{6D} + AE\tnote{1} & 0.10 & 41.6 & 0.40 & 27.7 & 76.05\\
		& $\Matrix{A}$   & 0.12 & 100 & 0.72 & 100 & ---\\
		& $\Matrix{A}$ + DT\tnote{2} & \textbf{0.09} & 79.1 & \textbf{0.10} & 39.2 & \textbf{97.41} \\ \bottomrule
	\end{tabular}
    \begin{tablenotes}
    	\item[1] Thresholding based on $q = 1.0$. \item[2] Thresholding based on $q = 0.75$. \item[3] \% of corrupted images that are rejected.
     \end{tablenotes}
\label{tab:kitti_relative_rotation_stats}
\end{threeparttable}	
\end{table*}

\subsection{Visual Rotation-Only Egomotion Estimation: KITTI}
Third, we used our representation to learn relative rotation from sequential images from the KITTI odometry dataset \citep{geiger_vision_2013}. By correcting or independently estimating rotation, these learned models have the potential to improve classical visual odometry pipelines \citep{peretroukhin_dpc-net_2018}. We note that even in the limit of no translation, camera rotation can be estimated independently of translation from a pair of images \citep{kneip_finding_2012}. To this end, we built a convolutional neural network that predicted the relative camera orientation between sequential images recorded on sequences from the \textit{residential} and \textit{city} categories in the dataset. We selected sequences \texttt{00}, \texttt{02}, and \texttt{05} for testing and trained three models with the remaining sequences for each. In accord with the results in \Cref{fig:synthetic_rotangle_box}, we found that there was little change in performance across different $\LieGroupSO{3}$ representations since rotation magnitudes of regression targets from KITTI are typically on the order of one degree. However, we found that our DT metric acted as a useful measure of epistemic uncertainty. To further validate this notion, we manually corrupted the KITTI test images by setting random rectangular regions of pixels to uniform black. \Cref{fig:kitti_box_02} displays the growth in magnitude that is manifest in the DT metric as data becomes less similar to that seen during training. \Cref{fig:kitti_egt_scatter} displays the estimation error for test sequence \texttt{02} with and without corruption. We stress that these corruptions are only applied to the test data; as we highlight numerically in \Cref{tab:kitti_relative_rotation_stats}, DT is able to reject corrupted images and other images that are likely to lead to high test error. Indeed, in all three test sequences, we observed a nearly constant mean rotation error for our formulation ($\Matrix{A}$ + DT) with and without data corruption.

\subsubsection{Auto-Encoder (AE) OOD Method}
We compared our DT thresholding approach with an auto-encoder-based OOD technique inspired by work in novelty detection in mobile robotics \cite{amini_variational_2018,richter_safe_2017}. For each training set, we trained an auto-encoder using an $L_1$ pixel-based reconstruction loss, and then rejected test-time inputs whose mean pixel reconstruction error is above the $q$th percentile in training. \Cref{fig:kitti_bar_errors} and \Cref{tab:kitti_relative_rotation_stats} detail results that demonstrate that our representation paired with DT performs significantly better than the \texttt{6D} representation with an AE-based rejection method. Importantly, we stress that our notion of uncertainty is embedded within the representation and does not require the training of an auxiliary OOD classifier.

\subsection{MAV Indoor-Outdoor Dataset}

\begin{figure}[t]	
	\centering
	\begin{subfigure}[]{\columnwidth}
\includegraphics[width=\linewidth]{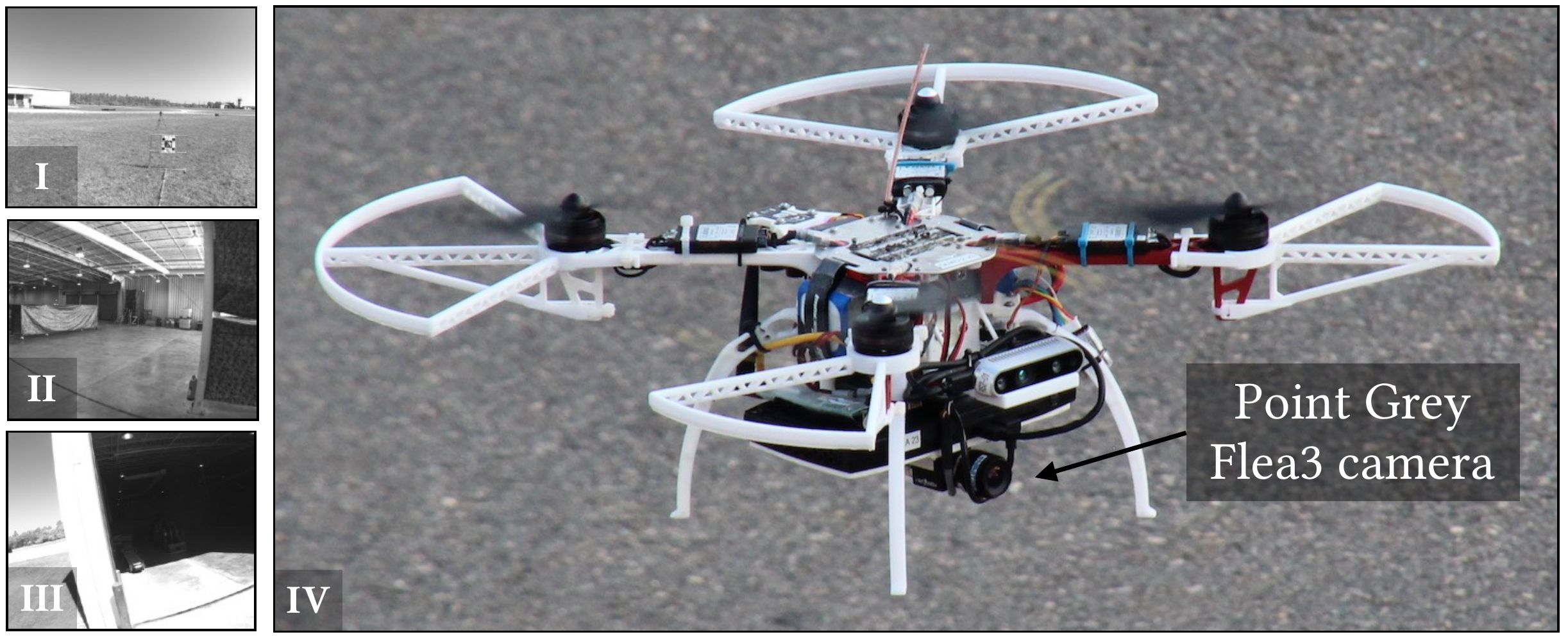}
\caption{MAV with a Point Grey Flea3 global shutter camera (IV) in three environments: outdoor (I), indoor (II) and transition (III).}
\label{fig:fla_drone}
	\end{subfigure}
	\begin{subfigure}[]{\columnwidth}
\includegraphics[width=\linewidth]{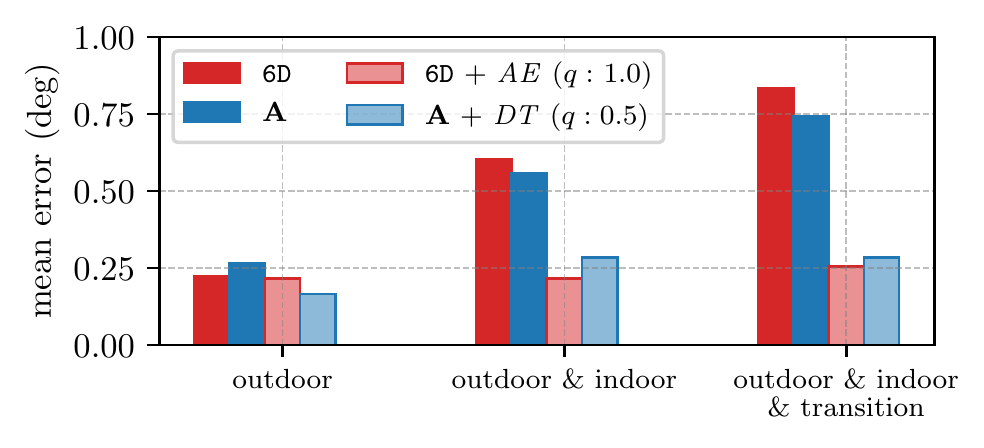}
\caption{Mean rotation errors for three different test sequences from the MAV dataset, using a model trained on outdoor data. The auto-encoder rejection performs well, yet our DT technique is able to match its performance without requiring a separate model.}
\label{fig:fla_bar}
	\end{subfigure}
\caption{A summary of our MAV experiments.}
\end{figure}

Finally, we applied our representation to the task of training a relative rotation model on data collected using a Flea3 global-shutter camera mounted on the Micro Aerial Vehicle depicted in \Cref{fig:fla_drone}. We considered a dataset in which the vehicle undergoes dramatic lighting and scene changes as it transitions from an outdoor to an indoor environment. We trained a model using outdoor images with ground-truth rotation targets supplied by an onboard visual-inertial odometry system (note that since we were primarily interested in characterizing our dispersion thresholding technique, such coarse rotation targets sufficed), and an identical network architecture to that used in the KITTI experiments. \Cref{fig:fla_bar} details the performance of our representation against the \texttt{6D} representation paired with an auto-encoder OOD rejection method. We observe that, compared to the KITTI experiment, the AE-based OOD rejection technique fares much better on this data. We believe this is a result of the way we partitioned our MAV dataset; images were split into a train/test split using a random selection, so test samples were recorded very near (temporally and spatially) to training samples. Nevertheless, our method, $\Matrix{A}$ + DT, performs on par with \texttt{6D} + AE on all three test sets, but does not require the training of a separate classifier.

\section{Discussion and Limitations}
Rotation representation is an important design criterion for state estimation and no single representation is optimal in all contexts; ours is no exception. Importantly, the differentiable layer which solves \Cref{prob:quaternion_qcqp} incurs some computational cost. This cost is negligible at test-time but can slow learning during training when compared to other representations that require only basic operations like normalization. In practice, we find that for common convolutional networks, training is bottlenecked by other parts of the learning pipeline and our representation adds marginal processing time. For more compact models, however, training time can be increased. Further, our representation does not include symmetric matrices where the minimal eigenvalue is non-simple. In this work, we do not enforce this explicitly; instead, we assume that this will not happen to within machine precision. In practice, we find this occurs exceedingly rarely, though explicitly enforcing this constraint  is a direction of future research.

\section{Conclusions and Future Work} \label{sec:conclusion}

In this work, we presented a novel representation of $\LieGroupSO{3}$ based on a symmetric matrix $\Matrix{A}$. Our representation space can be interpreted as a data matrix of a QCQP, as defining a Bingham belief over unit quaternions, or as parameterizing a weighted rotation average over a set of base rotations. Further, we proved that this representation admits a smooth global section of $\LieGroupSO{3}$ and developed an OOD rejection method based solely on the eigenvalues of $\Matrix{A}$. Avenues for future work include combining our representation with a Bingham likelihood loss, and further investigating the connection between $\Matrix{A}$, epistemic uncertainty, and rotation averaging. Finally, we are especially interested in leveraging our representation to improve the reliability and robustness (and ultimately, safety) of learned perception algorithms in real-world settings.

{\footnotesize
\bibliographystyle{plainnat}
\bibliography{robotics_abbrv,references}
}
\end{document}

%% file: packages.tex
\usepackage{times}

\usepackage[numbers,sort&compress]{natbib}
\usepackage{multicol}
\usepackage[bookmarks=true]{hyperref}

\usepackage{xparse} 
\usepackage{amsmath} 
\usepackage{amssymb}
\usepackage{amsfonts} 
\usepackage{dsfont}
\usepackage{graphicx}
\usepackage{url}
\usepackage{booktabs}
\usepackage[para]{threeparttable}
\usepackage{multirow}
\usepackage{hyperref}
\hypersetup{colorlinks,linkcolor={blue},citecolor={green},urlcolor={blue}}  
\usepackage{amsthm}
\usepackage{cleveref}
\usepackage{xcolor}
\usepackage{framed}
\usepackage{caption}
\usepackage{subcaption}

\usepackage{algorithmicx}
\usepackage{algorithm}
\usepackage[noend]{algpseudocode}

\algrenewcommand\algorithmicrequire{\textbf{Input:}}
\algrenewcommand\algorithmicensure{\textbf{Output:}}
\algnewcommand\algorithmicinput{\textbf{Input:}}
\algnewcommand\INPUT{\item[\algorithmicinput]}
\usepackage{setspace}

\colorlet{shadecolor}{gray!10}
   
\newtheorem{problem}{Problem}

\newtheorem{theorem}{Theorem}

\newcommand\blfootnote[1]{%
  \begingroup
  \renewcommand\thefootnote{}\footnote{#1}%
  \addtocounter{footnote}{-1}%
  \endgroup
}

%% file: notation.tex
\NewDocumentCommand\bbm{}{ \begin{bmatrix} }
\NewDocumentCommand\ebm{}{ \end{bmatrix} }
\NewDocumentCommand\Vector{m}{ \boldsymbol{\mathbf{#1}} }
\NewDocumentCommand\Matrix{m}{ \boldsymbol{\mathbf{#1}} }

\NewDocumentCommand\Trace{m}{ \mathrm{tr}\left(#1\right) }

\NewDocumentCommand\Norm{m}{\left\Vert#1\right\Vert }

\NewDocumentCommand\PartialDerivative{mm}{ \frac{\partial #1}{\partial #2} }

\NewDocumentCommand\ArgMin{m}{ \operatorname*{argmin}_{#1} }

\NewDocumentCommand\Real{}{ \mathbb{R} }
\NewDocumentCommand\RealProjective{m}{ \mathbb{RP}^{#1}}

\NewDocumentCommand\Sym{}{ \mathbb{S} }
\NewDocumentCommand\SymmetricMatrices{m}{\Sym^{#1}}

\NewDocumentCommand\LieGroupSO{m}{ \mathrm{SO}(#1) }

\NewDocumentCommand\LieGroupSE{m}{ \mathrm{SE}(#1) }

\NewDocumentCommand\MatLog{m}{\mathrm{Log}\left({#1}\right)}

\NewDocumentCommand\MatExp{m}{\mathrm{Exp}\left({#1}\right)}

\NewDocumentCommand\NormalDistribution{mm}{ \mathcal{N}\left(#1,#2\right) }

\NewDocumentCommand\Identity{}{ \Matrix{I} }

\NewDocumentCommand\Rotation{}{ \Matrix{R} }

\NewDocumentCommand\Quaternion{}{ \Vector{q} }

\NewDocumentCommand\Mean{m}{\overline{#1}}

\NewDocumentCommand\Defined{}{\triangleq}
\NewDocumentCommand\T{}{\mathsf{T}}

\NewDocumentCommand\QuatProd{}{\odot}
\NewDocumentCommand\QuatMat{}{\Matrix{\Omega}}

\NewDocumentCommand\KronProd{}{\otimes}
\NewDocumentCommand\diag{m}{\text{diag}\left(#1\right)}
\NewDocumentCommand\BinghamDistribution{mm}{\textsc{Bingham}\left(#1, #2\right)} 

\NewDocumentCommand\QuatDist{mm}{d_{\text{quat}}\left(#1, #2\right)} 
\NewDocumentCommand\ChordDist{mm}{d_{\text{chord}}\left(#1, #2\right)} 
\NewDocumentCommand\AngDist{mm}{d_{\text{ang}}\left(#1, #2\right)} 
\NewDocumentCommand\dc{}{\lambda^{\text{dc}}} 
\NewDocumentCommand\eigval{}{\lambda} 

\NewDocumentCommand\ShapeNet{}{\textsc{ShapeNet}~} 